\documentclass[final,11pt,noinfoline]{imsart}

\arxiv{}

\usepackage[english]{babel}
\usepackage[latin1]{inputenc}
\usepackage{textcomp}
\usepackage{dsfont}
\usepackage{color}
\usepackage{xcolor}
\usepackage{graphics}
\usepackage{graphicx}
\usepackage{epstopdf}
\usepackage{amsmath,amsthm,amssymb,amsfonts}
\usepackage{mathrsfs}
\usepackage[garamont]{mathdesign}
\usepackage{enumitem}
\usepackage[norelsize,ruled,vlined,commentsnumbered]{algorithm2e}
\usepackage{multirow}
\usepackage[final,activate,verbose=true,auto=true]{microtype}
\usepackage{a4wide}
\usepackage{bbm}
\usepackage{mathtools}
\usepackage{hyperref}
\usepackage{algorithmic}         
\usepackage{algorithm2e} 
\usepackage[mathscr]{eucal}

\hypersetup{colorlinks=true,linkcolor=blue,citecolor=blue,linktoc=page}

\newcommand{\eq}{\begin{equation}}
\newcommand{\qe}{\end{equation}}

\DeclareMathOperator*{\argmax}{arg\,max}

\newtheorem{theorem}{Theorem}
\newtheorem{lemma}[theorem]{Lemma}
\newtheorem{cor}[theorem]{Corollary}
\newtheorem{proposition}[theorem]{Proposition}

\newtheorem{remark}{Remark}

\date{\today}

\begin{document}
\sloppy

\begin{frontmatter}

\title{Latent Distance Estimation for 
           Random Geometric Graphs \thanks{Accepted to NeurIPS 2019}}
\runtitle{ Latent Distance Estimation for 
           Random Geometric Graphs}
\begin{aug}
\author{\fnms{Ernesto} \snm{Araya Valdivia}\ead[label=e1]{ernesto.araya-valdivia@math.u-psud.fr}}
\affiliation{Université Paris-Sud}
\address{Laboratoire de Math\'ematiques d'Orsay (LMO)\\ Universit\'e Paris-Sud \\ 91405 Orsay Cedex \\France}
\author{\fnms{Yohann} \snm{De Castro}\ead[label=e1]{Yohann.de-castro@enpc.fr}}
\affiliation{Ecole des Ponts ParisTech-CERMICS}
\address{Ecole des Ponts ParisTech-CERMICS\\6 et 8 avenue Blaise Pascal, Cit\'e Descartes\\ Champs sur Marne, 77455 Marne la Vall\'ee, Cedex 2\\France}
\runauthor{Araya-De Castro}

\end{aug}

\begin{abstract}
Random geometric graphs are a popular choice for a latent points generative model for networks. Their definition is based on a sample of $n$ points $X_1,X_2,\cdots,X_n$ on the Euclidean sphere~$\mathbb{S}^{d-1}$ which represents the latent positions of nodes of the network. The connection probabilities between the nodes are determined by an unknown function (referred to as the ``link'' function) evaluated at the distance between the latent points. We introduce a spectral estimator of the pairwise distance between latent points and we prove that its rate of convergence is the same as the nonparametric estimation of a function on $\mathbb{S}^{d-1}$, up to a logarithmic factor. In addition, we provide an efficient spectral algorithm to compute this estimator without any knowledge on the nonparametric link function. As a byproduct, our method can also consistently estimate the dimension $d$ of the latent space.

\end{abstract}



\end{frontmatter}

\maketitle
\section{Introduction}
\label{submission}

Random geometric graph (RGG) models have received attention lately as alternative to some simpler yet unrealistic models as the ubiquitous Erd\"os-R\'enyi model \cite{Erdos}. They are generative latent point models for graphs, where it is assumed that each node has associated a latent point in a metric space (usually the Euclidean unit sphere or the unit cube in $\mathbb{R}^d$) and the connection probability between two nodes depends on the position of their associated latent points. In many cases, the connection probability depends only on the distance between the latent points and it is determined by a one-dimensional ``link'' function.  

Because of its geometric structure, this model is appealing for applications in wireless networks modeling \cite{Jia}, social networks \cite{Hoff} and biological networks \cite{protein}, to name a few. In many of these real-world networks, the probability that a tie exists between two agents (nodes) depends on the similarity of their profiles. In other words, the connection probability depends on some notion of distance between the position of the agents in a metric space, which in the social network literature has been called the \emph{social space}. 

In the classical RGG model, as introduced by Gilbert in \cite{Gilbert}, we consider $n$ independent and identically distributed latent points $\{X_i\}^n_{i=1}$ in $\mathbb{R}^d$ and the  construct the graph with vertex set $V=\{1,2,\cdots,n\}$, where the node $i$ and $j$ are connected if and only if the Euclidean distance $\|X_i-X_j\|_d$ is smaller that certain predefined threshold $\tau$. The classic reference on the classical RGG model, from the probabilistic point-view, is the monograph \cite{Pen}. Another good reference is the survey paper \cite{Walters}. In that case, the ``link" function, which we have not yet formally defined, is the \emph{threshold} function $\mathbbm{1}_{t\leq \tau}(t)$. That is, the connection probability between two points is one or zero depending if their distance is smaller or larger than $\tau$. In that case, all the randomness lies in the fact that we are sampling the latent points with a certain distribution. We choose to maintain the name of random geometric graphs for more general ``link" functions.


We are interested in the problem of recovering the pairwise distances between the latent points $\{X_i\}^n_{i=1}$ for geometric graphs on the sphere $\mathbb{S}^{d-1}$ given an single observation of the network. We limit ourselves to the case when the network is a simple graph. Furthermore, we will assume that the dimension $d$ is fixed and that the ``link" function is not known. This problem and some related ones has been studied for different versions of the model and under a different set of hypothesis, see for example the recent work \cite{AriasCastro} and the references therein. In that work the authors propose a method for estimating the latent distances based on the graph theoretic distance between two nodes (that is the length of the shortest path that start in one node and finish on the other). Independently, in \cite{Mcdiarmid} the authors develop a similar approach which has slightly less recovery error, but for a less general model. In both cases, the authors consider the cube in $\mathbb{R}^d$ (or the whole $\mathbb{R}^{d}$) but not the sphere. Our strategy is similar to the one developed in \cite{Sussman}, where they considered the latent point estimation problem in the case of \emph{random dot product graphs}, which is a more restricted model compared to the one considered here. However, they considered more general Euclidean spaces and latent points distributions other than the uniform. Similar ideas has been used in the context vertex classification for latent position graphs \cite{TangSuss}.  

We will use the notion of graphon function to formalize the concept of ``link" function.  
Graphons are central objects to the theory of dense graph limits. They were introduced by Lov\'asz and Szegedy in \cite{Lovsze} and further developed in a series of papers, see \cite{Lova1},\cite{Lova2}. Formally, they are symmetric kernels that take values in $[0,1]$, thus they will act as the ``link'' function for the latent points. The spectrum of the graphon is defined as the spectrum of an associated integral operator, as in \cite[Chap.7]{Lov}. In this paper, they will play the role of limit models for the adjacency matrix of a graph, when the size goes to infinity. This is justified in light of the work of Koltchinskii and Gin\'e \cite{Kolt} and Koltchinskii \cite{Kolt2}. In particular, the adjacency matrix of the observed graph can be though as a finite perturbed version of this operator, combining results from \cite{Kolt} and \cite{BanVan}. 

We will focus on the case of dense graphs on the sphere $\mathbb{S}^{d-1}$ where the connection probability depends only on the geodesic distance between two nodes. This allows us to use the harmonic analysis on the sphere to have a nice characterization of the graphon spectrum, which has a very particular structure. More specifically, the following two key elements holds: first, the basis of eigenfunctions is fixed (do not depend on the particular graphon considered) and equal to the well-known spherical harmonic polynomials. Second, the multiplicity of each eigenvalue is determined by a sequence of integers that depends only on the dimension $d$ of the sphere and is given by a known formula and the associated eigenspaces are composed by spherical harmonics of the same polynomial degree.

The graphon eigenspace composed only with linear eigenfunctions (harmonic polynomials of degree one) will play an important role in the latent distances matrix recovery as all the information we need to reconstruct the distances matrix is contained in those eigenfunctions. We will prove that it is possible to approximately recover this information from the observed adjacency matrix of the graph under regularity conditions (of the Sobolev type) on the graphon and assuming an eigenvalue gap condition (similar hypotheses are made in \cite{Chatterjee} in the context of matrix estimation and in \cite{ManLearning} in the context of manifold learning). We do this by proving that a suitable projection of the adjacency matrix, onto a space generated by exactly $d$ of its eigenvectors, approximates well the latent distances matrix considering the mean squared error in the Frobenius norm. We give nonassymptotic bound for this quantity obtaining the same rate as the nonparametric rate of estimation of a function on the sphere $\mathbb{S}^{d-1}$, see \cite[Chp.2]{Nemirovski} for example. Our approach includes the adaptation of some perturbation theorems for matrix projections from the orthogonal to a ``nearly" orthogonal case, which combined with concentration inequalities for the spectrum gives a probabilistic finite sample bound, which is novel to the best of our knowledge. Our method share some similarities with the celebrated UVST method, introduced by Chatterjee in \cite{Chatterjee}, but in that case we obtain an estimator of the probability matrix described in Section \ref{sec:model} and not of the population Gram matrix as our method. We develop an efficient algorithm, which we call Harmonic EigenCluster(HEiC) to reconstruct the latent positions form the data and illustrate its usefulness with synthetic data.

\section{Preliminaries}

\subsection{Notation}
We will consider $\mathbbm{R}^d$ with the Euclidean norm $\|\cdot\|$ and the Euclidean scalar product $\langle\,,\,\rangle $. We define the sphere $\mathbb{S}^{d-1}:=\{
x \in \mathbbm{R}^d:\|x\|=1\}$. For a set $A\subset \mathbb{R}$ its diameter $diam(A):=~\sup_{x,y\in A}{|x-y|}$ and if $B\subset\mathbb{R}$ the distance between $A$ and $B$ is $dist(A,B)~:=\inf_{x\in A,y\in B}|x-y|$.
We will use $\|\cdot\|_F$ the Frobenius norm for matrices and $\|\cdot\|_{op}$ for the operator norm. The identity matrix in $\mathbb{R}^{d\times d}$ will be $\mathrm{Id}_{d}$. If $X$ is a real valued random variable and $\alpha\in(0,1)$, $X\leq_\alpha C$ means that $\mathbb{P}(X\leq C)\geq 1-\alpha$. 

\subsection{Generative model}\label{sec:model}

We describe the generative model for networks which is a generalization of the classical random geometric graph model introduced by Gilbert in \cite{Gilbert}. We base our definition on the $W$-random graph model described in \cite[Sec. 10.1]{Lov}. The central objects will be graphon functions on the sphere, which are symmetric measurable functions of the form $W~:~\mathbb{S}^{d-1}\times \mathbb{S}^{d-1}\rightarrow [0,1]$. Throughout  this paper, we consider the measurable space $(\mathbb{S}^{d-1},\sigma)$, where $\sigma$ is the uniform measure on the sphere. On $\mathbb{S}^{d-1}\times \mathbb{S}^{d-1}$ we consider the product measure $\sigma\times \sigma$. 

Now we describe how to generate a simple graph with $n$ nodes from a graphon function $W$ and a sample of points on the sphere $\{X_i\}^n_{i=1}$, known as the latent points. 

First, we sample $n$ points $\{X_i\}^n_{i=1}$ independently on the sphere $\mathbb{S}^{d-1}$, according to the uniform measure $\sigma$. These are the so-called latent points. Second, we construct the matrix of distances between these points, called the \textit{Gram matrix} $\mathcal{G}^\ast$ (we will often call it population Gram matrix) defined by\[\mathcal{G}^\ast_{ij}:=\langle X_i,X_j \rangle\]
and the so-called \emph{probability matrix} \[\Theta_{ij}=\rho_n W( X_i, X_j)\] which is also a $n\times n$ matrix. The function $W$ gives the precise meaning for the ``link'' function, because it determines the connection probability between $X_i$ and $X_j$. The introduction of the scale parameter $0<\rho_n\leq 1$ allow us to control the edge density of the sampled graph given a function $W$, see \cite{Verzelen} for instance. The case $\rho_n=1$ corresponds to the dense case (the parameter $\Theta_{ij}$ do not depend on $n$) and  when $\rho_n\rightarrow 0$ the graph will be sparser. Our main results will hold in the regime $\rho_n=\Omega(\frac{\log{n}}{n})$, which we call \emph{relatively sparse}. Most of the time we will work with the normalized version of the probability matrix $T_n:=\frac{1}{n}\Theta$. If there exists a function $f:[-1,1]\rightarrow [0,1]$ such that $W(x,y)=f(\langle x,y\rangle)$ for all $x,y\in\mathbb{S}^{d-1}$ we will say that $W$ is a geometric graphon. 

Finally, we define the random adjacency matrix $\hat{T}_n$, which is a $n\times n$ symmetric random matrix that has independent entries (except for the symmetry constraint $\hat{T}_n=\hat{T}_n^T$), conditional on the probability matrix, with laws \[n(\hat{T}_n)_{ij}\sim \mathcal{B}(\Theta_{ij})\] where $\mathcal{B}(m)$ is the Bernoulli distribution with mean parameter $m$. Since the probability matrix contains the mean parameters for the Bernouilli distributions that define the random \emph{adjacency} matrix it has been also called the \emph{parameter matrix} \cite{Chatterjee}. Observe that the classical RGG model on the sphere is a particular case of the described $W$-random graph model when $W(x,y)=\mathbbm{1}_{\langle x,y\rangle\leq \tau}$. In that case, since the entries of the probability matrix only have values in $\{0,1\}$, the adjacency matrix and the probability matrix are equal. Depending on the context, we use $\hat{T}_n$ for the random matrix as described above or for an instance of this random matrix, that is for the adjacency matrix of the observed graph. This will be clear from the context.

Thus the generative model can be seen as a two step sampling procedure where first the latent points are generated (which determine the Gram matrix and the probability matrix) and conditional on those points we generate the adjacency matrix. 

It is worth noting that graphons can be, without loss of generality, defined in $[0,1]^2$. The previous affirmation means that for any graphon there exists a graphon in $[0,1]^2$ that generates the same distribution on graphs for any given number of nodes. However, in many cases the $[0,1]^2$ representation can be less revealing than other representations using a different underlying space. This is illustrated in the case of the \emph{prefix attachment} model in \cite[example 11.41]{Lov}.

In the sequel we use the notation $\lambda_0,\lambda_1,\cdots,\lambda_{n-1}$ for the eigenvalues of the normalized probability matrix $T_n$. 
Similarly, we denote by $\hat{\lambda}_0,\hat{\lambda}_1,\cdots,\hat{\lambda}_{n-1}$ the eigenvalues of the matrix $\hat{T}_n$. 
We recall that $T_n$ (resp. $\hat{T}_n$) and $\frac{1}{\rho_n}T_n$ (resp.$\frac{1}{\rho_n}\hat{T}_n$ ) have the same set of eigenvectors.
We will denote by $v_j$ for $1\leq j\leq n$ the eigenvector of $T_n$ associated to $\lambda_j$, which is also the eigenvector of $\frac{1}{\rho_n}T_n$ associated to $\frac{1}{\rho_n}\lambda_j$. Similarly, we denote by $\hat{v}_j$ to the eigenvector associated to the eigenvalue $\rho_n\hat{\lambda}_j$ of $\hat{T}_n$.

Our main result is that we can recover the Gram matrix using the eigenvectors of $\hat{T}_n$ as follows
\begin{theorem}[Informal statement]\label{thm:informal}
There exists a constant~$c_1>0$ that depends only on the dimension $d$ such that the following is true. Given a graphon $W$ on the sphere such that $W( x,y )=f(\langle x,y \rangle)$ with $f:[-1,1]\rightarrow [0,1]$ unknown, which satisfies an eigenvalue gap condition and has Sobolev regularity $s$, there exists a subset of the eigenvectors of $\hat{T}_n$, such that $\hat{\mathcal G}:=\frac1{c_1}\hat V\hat V^T$ converges to the population Gram matrix~$\mathcal{G}^\ast:=\frac1n(\langle X_i,X_j\rangle)_{i,j}$ at rate $n^{\frac{-s}{2s+d-1}}$ (up to a log factor). This estimate $\hat V\hat V^T$ can be found in linear time given the spectral decomposition of $\hat{T}_n$.
\end{theorem}
We will say that a geometric graphon $W(x,y)=f(\langle x, y\rangle)$ on $\mathbb{S}^{d-1}$ has regularity $s$ if $f$ belongs the Weighted Sobolev space $Z^s_\gamma([-1,1])$ with weight function $w_{\gamma}(t)=(1-t)^{\gamma-\frac{1}{2}}$, as defined in \cite{Nica}. 
In order to make the statement of \ref{thm:informal} rigorous, we need to precise the eigenvalue gap condition and define the graphon eigensystem.
\subsection{Geometric graphon eigensystem}\label{sec:eigen}
Here we gather some asymptotic and concentration properties for the eigenvalues and eigenfunctions of the matrices $\hat{T}_n,T_n$ and the operator $T_W$, which allows us to recover the Gram matrix from data. The key fact is that the eigenvalues (resp. eigenvectors) of the matrix $\frac{1}{\rho_n}\hat{T}_n$ and $\frac{1}{\rho_n}T_n$ converge to the eigenvalues (resp. sampled eigenfunctions) of the integral operator $T_W:L^2(\mathbb{S}^{d-1})\rightarrow L^2(\mathbb{S}^{d-1})$ 
\[T_W g(x)=\int_{\mathbb{S}^{d-1}}g(y)W(x,y)d\sigma(y)\]
which is compact \cite[Sec.6, example 1]{Hirsch} and self-adjoint (which follows directly from the symmetry of $W$). Then by a classic theorem in functional analysis \cite[Sec.6, Thm. 1.8]{Hirsch} its spectrum is a discrete set $\{\lambda^\ast_k\}_{
k\in\mathbb{N}}\subset\mathbb{R}$ and its only accumulation point is zero. In consequence, we can see the spectra of $\hat{T}_n$, $T_n$ and $T_W$ (which we denote $\lambda(\hat{T}_n)$, $\lambda(T_n) $ and $\lambda(T_W)$ resp.) as elements of the space $\mathcal{C}_0$ of infinite sequences that converge to $0$ ( where we complete the finite sequences with zeros). It is worth noting that in the case of geometric graphons with regularity $s$ (in the Sobolev sense defined above) the rate of convergence of $\lambda(T_W)$ is determined by the regularity parameter $s$. We have the following:
\begin{itemize}
\item The spectrum of $\lambda(\frac{1}{\rho_n}T_n)$ converges to $\lambda(T_W)$ (almost surely) in the $\delta_2$ metric, defined as follows\[\delta_2(x,y)=\inf_{p\in \mathcal{P}}\sqrt{\sum_{i\in\mathbb{N}}(x_i-y_{p(i)})^2 }\] where $\mathcal{P}$ is the set of all permutations of the non-negative integers. This is proved in \cite{Kolt}.
\item Matrices $\hat{T}_n$ approach to matrix $T_n$ in operator norm as $n$ gets larger. Applying \cite[Cor.3.3]{BanVan} to the centered matrix $Y=\hat{T}_n-T_n$ we get \begin{equation}\label{eq:banvanesp}\mathbb{E}(\|\hat{T}_n-T_n\|_{op})\lesssim \frac{D_0}{n}+\frac{D_0^\ast\sqrt{\log{n}}}{n}\end{equation} where $\lesssim$ denotes inequality up to constant factors, $D_0=\max_{0\leq i\leq n}\sum^n_{j=1}Y_{ij}(1-Y_{ij})$ and $D_0^\ast=\max_{ij}|Y_{ij}|$. We clearly have that $D_0=\mathcal{O}(n\rho_n)$ and $D_0^\ast\leq 1$, which implies that \[\mathbb{E} \|\hat{T}_n-T_n\|_{op}\lesssim \max{\Big\{\frac{\rho_n}{\sqrt n},\frac{\sqrt{\log{n}}}{n}\Big\}}\]
We see that this inequality do not improve if $\rho_n$ is smaller than in the relatively sparse case, that is $\rho_n=\Omega(\frac{\log{n}}{n})$. A similar bound can be obtained for the Frobenius norm replacing $\hat{T}_n$ with $\hat{T}^{\mathrm{uvst}}_n$ the UVST estimator defined in \cite{Chatterjee}. For our main results, Proposition \ref{prop:mainprop} and Theorem \ref{thm:mainthm} the operator norm bound will suffice. 
\end{itemize}
A remarkable fact in the case of geometric graphons on $\mathbb{S}^{d-1}$, that is when $W(x,y)=f(\langle x,y \rangle)$, is that the eigenfunctions $\{\phi_k\}_{k\in\mathbb{N}}$ of the integral operator $T_W$ are a fixed set that do not depend on the particular function $f$ on consideration. This comes from the fact that $T_W$ is a convolution operator on the sphere and its eigenfunctions are the well known \textit{spherical harmonics} of dimension $d$, which are harmonic polynomials in $d$ variables defined on $\mathbb{S}^{d-1}$ corresponding to the eigenfunctions of the Laplace-Beltrami operator on the sphere. This follows from \cite[Thm.1.4.5]{Dai} and from the Funck-Hecke formula given in \cite[Thm.1.2.9]{Dai}. Let $d_k$ denote the dimension of the $k$-th spherical harmonic space. It is well known \cite[Cor.1.1.4]{Dai} that $d_0=1$, $d_1=d$ and $d_k=\binom{k+d-1}{k}-\binom{k+d-3}{k-2}$. Another important fact, known as the \emph{addition theorem} \cite[Lem.1.2.3 and Thm.1.2.6]{Dai}, is that \[\sum^{d_{k}}_{i=d_{k-1}}\phi_{j}(x)\phi_{j}(y)=c_k G^{\gamma}_k(\langle x,y \rangle)\]
where $G^\gamma_k$ are the Gegenbauer polynomials of degree $k$ with parameter $\gamma=\frac{d-2}{2}$ and $c_k=\frac{2k+d-2}{d-2}$.   

The Gegenbauer polynomial of degree one is $G^\gamma_1(t)=2\gamma t$ (see \cite[Appendix B2]{Dai}), hence we have $G^\gamma_1(\langle X_i,X_j\rangle)=2\gamma\langle X_i,X_j \rangle$ for every $i$ and $j$. In consequence, by the addition theorem 
\[G^\gamma_1(\langle X_i,X_j \rangle)=\frac{1}{c_1}\sum^{d}_{k=1}\phi_{k}(X_i)\phi_{k}(X_j)\] 
where we recall that $d_1=d$. This implies the following relation for the Gram matrix, observing that $2\gamma c_1=d$
\begin{equation}\label{eq:gramest}
\mathcal{G}^\ast:=\frac1n(\langle X_i,X_j\rangle)_{i,j}=\frac{1}{2\gamma c_1}\sum^{d}_{j=1}v^\ast_j{v^\ast_j}^T=\frac{1}{d}V^\ast {V^\ast}^T
\end{equation} where $v^\ast_j$ is the $\mathbbm{R}^n$ vector with $i$-th coordinate $\phi_{j}(X_i)/\sqrt n$ and $V^\ast$ is the matrix with columns $v^\ast_j$. In a similar way, we define for any matrix $U$ in $\mathbbm{R}^{n\times d}$ with columns $u_1,u_2,\cdots,u_d$, the matrix $\mathcal{G}_U:=\frac{1}{d}UU^T$.
As part of our main theorem we prove that for $n$ large enough there exists a matrix $\hat{V}$ in $\mathbbm{R}^{n\times d}$ where each column is an eigenvector of $\hat{T}_n$, such that $\hat{\mathcal{G}}:= \mathcal{G}_{\hat{V}}$ approximates $\mathcal{G}^\ast$ well, in the sense that the norm $\|\hat{\mathcal{G}}-\mathcal{G}^\ast\|_{F}$ converges to $0$ at a rate which is that of the non-parametric estimation of a function on $\mathbb{S}^{d-1}$.   

\subsection{Eigenvalue gap condition}
In this section we describe one of our main hypotheses on $W$ needed to ensure that the space $\operatorname{span} \{v^\ast_1,v^\ast_2,\cdots,v^\ast_d\}$ can be effectively recovered with the vectors $\hat{v}_1,\hat{v}_2,\cdots,\hat{v}_d$ using our algorithm. Informally, we assume that the eigenvalue $\lambda^\ast_1$ is sufficiently isolated from the rest of the spectrum of $T_W$. Given a geometric graphon $W$, we define the \emph{spectral gap} of $W$ relative to the eigenvalue $\lambda^\ast_1$ by \[\operatorname{Gap}_1(W):=\min_{j\notin \{1,\cdots,d_1\}}{|\lambda^\ast_1-\lambda^\ast_j|}\] which quantifies the distance between the eigenvalue $\lambda^\ast_1$ and the rest of the spectrum. In particular, we have the following elementary proposition.

\begin{proposition}
It holds that $\operatorname{Gap}_1(W)=0$ if and only if there exists $j\notin \{1,\cdots,d_1\}$ such that $\lambda^\ast_j=\lambda^\ast_1$ or $\lambda^\ast_1=0$.
\end{proposition}

\begin{proof}
Observe that the unique accumulation point of the spectrum of $T_W$ is zero. The proposition follows from this observation.
\end{proof}

To recover the population Gram matrix $\mathcal{G}^\ast$ with our Gram matrix estimator $\hat{\mathcal{G}}$ we require the spectral gap  $\Delta^\ast:=\operatorname{Gap}_1(W)$ to be different from $0$. This assumption have been made before in the literature, in results that are bases in some versin of the Davis-Kahan $\sin {\theta}$ theorem (see for instance \cite{Chatterjee} , \cite{ManLearning}, \cite{TangSuss}). More precisely, our results will hold on the following event \[\mathcal E:=\Big\{\delta_2\Big(\lambda\big(\frac{1}{\rho_n}T_n\big),\lambda(T_W)\Big)
\vee \frac{2^{\frac92}\sqrt d}{\rho_n\Delta^\ast}\|T_n-\hat{T_n}\|_{op}\leq\frac{ \Delta^\ast}4\Big\}\,,\]
for which we prove the following: given an arbitrary $\alpha$ we have that \[\mathbb{P}(\mathcal{E})\geq 1-\frac{\alpha}{2}\]
for $n$ large enough (depending on $W$ and $\alpha$). The following results are the main results of this paper. Their proofs can be found in the supplementary material.
\begin{proposition}\label{prop:mainprop}
On the event $\mathcal E$, there exists one and only one set $\Lambda_1$, consisting of $d$ eigenvalues of $\hat{T_n}$, whose diameter is smaller that $\rho_n\Delta^\ast/2$ and whose distance to the rest of the spectrum of $\hat{T_n}$ is at least $\rho_n\Delta^\ast/2$. Furthermore, on the event~$\mathcal E$, our algorithm (Algorithm \ref{alg:recovery2}) returns the  matrix $\hat{\mathcal G}=(1/c_1)\hat V\hat V^T$, where $\hat{V}$ has by columns the eigenvectors corresponding to the eigenvalues on $\Lambda_1$.
\end{proposition}
\begin{theorem}\label{thm:mainthm}
Let $W$ be a regular geometric graphon on $\mathbb{S}^{d-1}$, with regularity parameter $s$, such that $\Delta^\ast>0$. Then there exists a set of eigenvectors $\hat{v}_1,\hat{v}_2,\cdots,\hat{v}_d$ of $\hat{T}_n$ such that \[\|\mathcal{G}^\ast-\hat{\mathcal{G}}\|_F=O(n^{-\frac{s}{2s+d-1}})\]
where $\hat{\mathcal{G}}=\mathcal{G}_{\hat{V}}$ and $\hat{V}$ is the matrix with columns $\hat{v}_1,\hat{v}_2,\cdots,\hat{v}_d$. Moreover, this rate is the minimax rate of non-parametric estimation of a regression function $f$ with Sobolev regularity $s$ in dimension $d-1$.
\end{theorem}
The condition $\Delta^\ast>0$ allow us to use Davis-Kahan type results for matrix perturbation to prove Theorem \ref{thm:mainthm}. With this and concentration for the spectrum we are able to control with high probability the terms $\|\hat{\mathcal{G}}-\mathcal{G}\|_F$ and $\|\mathcal{G}-~\mathcal{G}^\ast\|_F$. The rate of nonparametric estimation of a function in $S^{d-1}$ can be found in \cite[Chp.2]{Nemirovski}.


\section{Algorithms}\label{sec:alg}
The Harmonic EigenCluster algorithm(HEiC) (see Algorithm \ref{alg:recovery2} below) receives the observed adjacency matrix $\hat{T}_n$ and the sphere dimension as its inputs to reconstruct the eigenspace associated to the eigenvalue $\lambda^\ast_1$. In order to do so, the algorithm selects $d$ vectors in the set $\hat{v}_1,\hat{v}_2,\cdots \hat{v}_n$, whose linear span is close to the span of the vectors $v^\ast_1,v^\ast_2,\cdots,v^\ast_d$ defined in Section \ref{sec:eigen}. The main idea is to find a subset of $\{\hat{\lambda}_0,\hat{\lambda}_2,\cdots,\hat{\lambda}_{n-1}\}$, which we call $\Lambda_1$, consisting on $d_1$ elements (recall that $d_1=d$) and where all its elements are close to $\lambda^\ast_1$. This can be done assuming that the event $\mathcal{E}$ defined above holds (which occurs with high probability). Once we have the set $\Lambda_1$, we return the span of the eigenvectors associated to the eigenvalues in $\Lambda_1$.

For a given set of indices $i_1,\cdots,i_d$ we define \[\operatorname{Gap}_1(\hat{T}_n; i_1,\cdots,i_d):=\min_{i\notin\{i_1,\cdots,i_d\}}\max_{j\in \{i_1,\cdots,i_j\}}|\hat{\lambda}_j-\hat{\lambda}_i|\] and 
\[\operatorname{Gap}_1(\hat{T}_n):=\max_{\{i_1,\cdots,i_d\}\in \mathcal{S}_d^n}{\operatorname{Gap}_1(\hat{T}_n; i_1,\cdots,i_d)}\] where $\mathcal{S}_d^n$ contains all the subsets of $\{1,\cdots,n-1\}$ of size $d$. This definition parallels that of $\operatorname{Gap}_1(W)$ for the graphon. Observe any set of indices in $\mathcal{S}_d^n$ will not include $0$. Otherwise stated, we can leave $\hat{\lambda}^{\mathrm{sort}}_0$ out of this definition and it will not be candidate to be in $\Lambda_1$. In the supplementary material we prove that the largest eigenvalue of the adjacency matrix will be close to the eigenvalue $\lambda^\ast_0$ and in consequence can not be close enough to $\lambda^\ast_1$ to be in the set $\Lambda_1$, given the definition of the event $\mathcal{E}$. 

  To compute $\operatorname{Gap}_1(\hat{T}_n)$ we consider the set of eigenvalues $\hat{\lambda}_j$ ordered in decreasing order. We use the notation $\hat{\lambda}^{\mathrm{sort}}_j$ to emphasize this fact. We define the right and left differences on the sorted set by \begin{align*}
    \mathrm{left}(i)&=|\hat{\lambda}^{\mathrm{sort}}_i-\hat{\lambda}^{\mathrm{sort}}_{i-1}|\\
    \mathrm{right}(i)&=\mathrm{left}(i+1)
\end{align*} where $\mathrm{left}(\cdot)$ is defined for $1\leq i\leq n$ and $\mathrm{right}(\cdot)$ is defined for $0\leq i\leq n-1$.
With these definition, we have the following lemma, which we prove in the supplementary material.
\begin{lemma}\label{lem:eq.defsgap}
On the event $\mathcal{E}$, the following equality holds
\[\operatorname{Gap}_1(\hat{T}_n)=\max{\Big\{\max_{1\leq i\leq n-d}{\min{\{\mathrm{left}(i),\mathrm{right}(i+d)\}}},\mathrm{left}(n-d+1)\Big\}}\]
\end{lemma}

The set $\Lambda_1$ has the form $\Lambda_1=\{\hat{\lambda}^{\mathrm{sort}}_{i^\ast},\hat{\lambda}^{\mathrm{sort}}_{i^\ast+1},\cdots,\hat{\lambda}^{\mathrm{sort}}_{i^\ast+d}\}$ for some $1\leq i^\ast\leq n-d$.
 We have that either \[i^\ast=\argmax_{1\leq i\leq n-d}{\min{\{\mathrm{left}(i),\mathrm{right}(i+d)\}}}\] or $i^\ast=n-d$ depending whether or not one has $\max_{1\leq i\leq n-d}{\min{\{\mathrm{left}(i),\mathrm{right}(i+d)\}}}>\mathrm{left}(n-d+1)$. The algorithm then constructs the matrix $\hat{V}$ having columns $\{\hat{v}_{i^\ast},\hat{v}_{i^\ast+1},\cdots,\hat{v}_{i^\ast+d}\}$ and returns $\hat{V}\hat{V}^T$.

It is worth noting that Algorithm \ref{alg:recovery2} time complexity $n^3+n$, where $n^3$ comes from the fact that we compute the eigenvalues and eigenvectors of the $n\times n$ matrix $\hat{T}_n$ and the linear term is because we explore the whole set of eigenvalues to find the maximum gap for the size $d$. In terms of space complexity the algorithm is $n^2$ because we need to store the matrix $\hat{T}_n$.
\begin{algorithm}[tb]
   \caption{Harmonic EigenCluster(HEiC) algorithm}
   \label{alg:recovery2}
\begin{algorithmic}
    \STATE {\bfseries Input:} $(\hat{T}_n,d)$ adjacency matrix and sphere dimension 
    \STATE $\Lambda^\mathrm{sort}=\{\hat{\lambda}^\mathrm{sort}_1,\cdots,\hat{\lambda}^\mathrm{sort}_{n-1}\} \leftarrow $eigenvalues of $\hat{T}_n$ sorted in decreasing order 
    \STATE $\Lambda_1\leftarrow \{\Lambda^\mathrm{sort}_1,\cdots,\Lambda^\mathrm{sort}_{1+d}\}$: where $\Lambda_i^\mathrm{sort}$ is the $i$-th element in $\Lambda^\mathrm{sort}$
    \STATE Initialize $i=2$, $\operatorname{gap}=\operatorname{Gap}_1(\hat{T}_n; 1,2,\cdots,d)$
    \WHILE{$i \leq n-d$}
    \IF{$ \operatorname{Gap}_1(\hat{T}_n; i,i+1,\cdots,i+d)>\operatorname{gap}$}
    \STATE $\Lambda_1\leftarrow \{\Lambda^\mathrm{sort}_i,\cdots,\Lambda^\mathrm{sort}_{i+d}\}$ 
    \ENDIF
    \STATE $i=i+1$
    \ENDWHILE
    \STATE {\bfseries Return:} $\Lambda_1$, $\operatorname{gap}$
\end{algorithmic}
\end{algorithm}

\begin{remark}
If we change $\hat{T}_n$ in the input of Algorithm \ref{alg:recovery2} to $\hat{T}^{\mathrm{usvt}}_n$ (obtained by the UVST algorithm \cite{Chatterjee}) we predict that the algorithm will give similar results. This is because discarding some eigenvalues bellow a prescribed threshold do not have effect on our method. However, as preprocessing step the UVST might help in speeding up the eigenspace detection, but this step is already linear in time. The study of the effect of UVST as preprocessing step is left for future work. 
\end{remark}

\subsection{Estimation of the dimension $d$}
So far we have focused on the estimation of the population Gram matrix $\mathcal{G}^\ast$. We now give an algorithm to find the dimension $d$, when it is not provided as input. This method receives the matrix $\hat{T}_n$ as input and uses Algorithm \ref{alg:recovery2} as a subroutine to compute a score, which is simply the value of the variable $\operatorname{Gap}_1(\hat{T}_n)$ returned by Algorithm \ref{alg:recovery2}. We do this for each $d$ in a set of candidates, which we call $\mathcal{D}$. This set of candidates will be usually fixed to $\{1,2,3,\cdots,d_{max}\}$. Once we have computed the scores, we pick the candidate that have the maximum score. 

Given the guarantees provided by Theorem \ref{thm:mainthm}, the previously described procedure will find the correct dimension, with high probability (on the event $\mathcal{E}$), if the true dimension of the graphon is on the candidate set $\mathcal{D}$. This will happen, in particular, if the assumptions of Theorem \ref{thm:mainthm} are satisfied. We recall that the main hypothesis on the graphon is that the spectral gap $\operatorname{Gap}_1(W)$ should be different from $0$. 

\section{Experiments}
We generate synthetic data using different geometric graphons. In the first set of examples, we focus in recovering the Gram matrix when the dimension is provided. In the second set we tried to recover the dimension as well. The Python code of these experiments is provided in the supplementary material. 

\subsection{Recovering the Gram matrix}
We start by considering the graphon $W_1(x,y)=\mathbbm{1}_{\langle x,y\rangle\leq 0}$ which defines, through the sampling scheme given in Section \ref{sec:model}, the same random graph model as the classical RGG model on $\mathbb{S}^{d-1}$ with threshold $0$. Thus two sampled points $X_i,X_j\in \mathbb{S}^{d-1}$ will be connected if and only if they lie in the same semisphere. 


\begin{figure}[h!]
\vskip 0.2in
\begin{center}
\begin{minipage}[t]{5cm}
\centering
\includegraphics[scale=0.25]{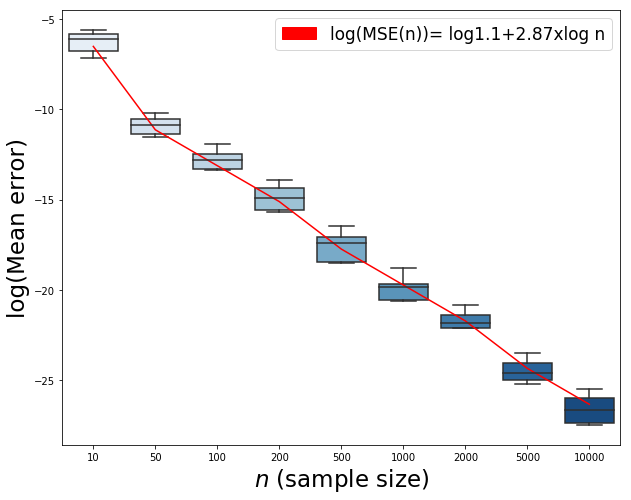}
\end{minipage}
\hspace{0.5cm}
\begin{minipage}[t]{5cm}
\centering
\includegraphics[scale=0.25]{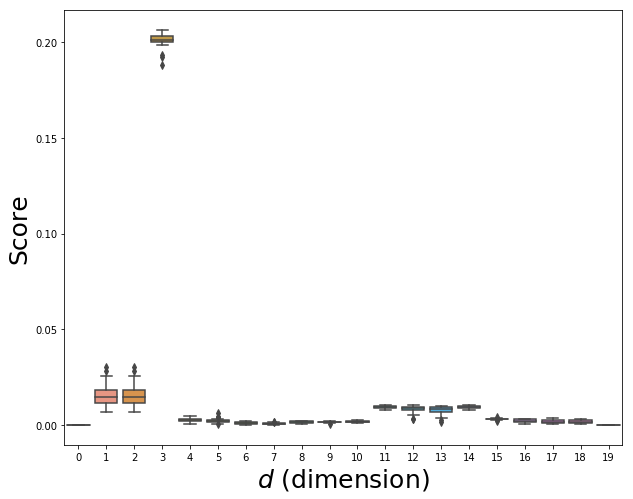}
\end{minipage}
\hspace{0.2cm}
\caption{In the left we have a boxplot of $MSE_n$ for different values of $n$. In the right, we plot the score for a set of candidate dimensions $\mathcal{D}=\{1,\cdots,15\}$. Data was sampled with $W_1$ on $\mathbb{S}^{d-1}$ with $d=3$.}
\label{fig:boxplot_final}
\end{center}
\vskip -0.2in
\end{figure}
We consider different values for the sample size $n$ and for each of them we sample $100$ Gram matrices in the case $d=3$ and run the Algorithm \ref{alg:recovery2} for each. We compute each time the mean squared error, defined by \[MSE_n=\frac{1}{n^2}\|\hat{\mathcal{G}}-\mathcal{G}^\ast\|^2_F\]


In Figure \ref{fig:boxplot_final} we put the $MSE_n$ for different values of $n$, showing how $MSE_n$ decrease in terms of $n$. For each $n$, the $MSE_n$ we plot is the mean over the $100$ sampled graphs. 



\subsection{Recovering the dimension $d$}

We conducted a simulation study using graphon $W_1$, sampling $1000$ point on the sphere of dimension $d=3$ and we use Algorithm \ref{alg:recovery2} to compute a score and recover $d$. We consider a set of candidates with $d_{max}=15$. In Figure \ref{fig:boxplot_final} we provide a boxplot for the score of each candidate repeating the procedure $50$ times. We see that for this graphon, the algorithm can each time differentiates the true dimension from the ``noise". We include more experiments in the supplementary material. 
\section{Discussion}
Although on this paper we have focused on the sphere as the latent metric space, our main result can be extended to other latent space where the distance is translation invariant, such as compact Lie groups or compact symmetric spaces. In that case, the geometric graphon will be of the form $W(x,y)=f(\cos{\rho(x,y)})$ where $x,y$ are points in the compact Lie group $\mathbbm{S}$ and $\rho(\cdot,\cdot)$ is the metric in this space. We will have \[f(\cos{\rho(x,y)})=f(\cos{\rho(x\cdot y^{-1},e_1)})=\tilde{f}(x\cdot y^{-1})\] where $e_1$ is the identity element in $\mathbbm{S}$ and $\tilde{f}(x)=f(\rho(x,e_1))$. In consequence $W(x,y)=\tilde{f}(x\cdot y^{-1})$. In addition, there exist an addition theorem in this case (which is central in our recovery result). Similar regularity notions to the one considered in this work also exist. They are related to rate of convergence to zero of the eigenvalues of integral operator associated to the graphon. In \cite{Yohann} the authors give more details on the model of geometric graphon in compact lie groups with focus on the estimation of the graphon function.

\bibliography{bibGramMat}
\bibliographystyle{plain}

\appendix
\section{Graphon regularity}\label{regularity}
One way to define the regularity of a geometric graphon on $\mathbb{S}^{d-1}$ is through the notion of weighted Sobolev spaces on the interval $[-1,1]$. In that context, the regularity is related to the rate at which the eigenvalue sequence $\{\lambda^\ast_i\}^\infty_{i=0}$ tends to $0$. Here we follow \cite{Nica}. For a function of the form $f(t)=\sum_{k\geq 0}\mu_kc_kG^\gamma_k(t)$, we define the norm \[\|f\|^2_{Z^s_\gamma}=\sum^\infty_{k=0}d_k|\mu_k|^2\big(1+k(k+2\gamma+1))^s\big)\]
We will say that $f$ belongs to weighted Sobolev space $Z^s_\gamma$ if $\|f\|_{Z^s_\gamma}\leq \infty$. We will refer to $s$ as the regularity parameter. As in the case of classical Sobolev spaces, there is a definition of weighted Sobolev spaces that involves the integrability (with respect to the measure $w_\gamma(t)dt$) of the weak derivatives of a function. Then a function $f$ belongs to $Z^s_\gamma$ if it has $s$ weak derivatives that are integrable with respect to the weighted $L^2$ norm in $[-1,1]$ with weight $w_\gamma(t)=(1-t)^{\gamma-\frac{1}{2}}$. In \cite{Nica} the authors prove that both definitions are in fact equivalent.

\section{Geometric Graphons have $\lambda^\ast_0$ as the largest eigenvalue}\label{sec:largest}
To avoid border issues in Algorithm $1$, we use the fact that the eigenvalue $\mu^\ast_0$ associated to the Gegenbauer polynomial $G^\gamma_0(t)=\mathbbm{1}(t):=1$ for $t\in [-1,1]$ is the largest one, which in the notation of the paper can be written as $\lambda^{\mathrm{sort}}_0=\mu^\ast_0$. This is true for all geometric graphons.
\begin{lemma}\label{lem:lambda0largest}
If $W:\mathbb{S}^{d-1}\times \mathbb{S}^{d-1}\rightarrow [0,1]$ is such that \[W(x,y)=f(\langle x,y \rangle)\] for $f:[-1,1]\rightarrow [0,1]$, then \[d_W(x):=\int_{\mathbb{S}^{d-1}}W(x,y)d\sigma(y)\] is constant. 
\end{lemma}
\begin{proof}
The proof follows from a change of variable.
\end{proof}

The following theorem is an analogous result to a classical theorem of spectral graph theory
\begin{theorem}\label{thm:lambda0largest}
For a graphon $W:\mathbb{S}^{d-1}\times \mathbb{S}^{d-1}\to [0,1]$ we have 
\[
\int_{\mathbb{S}^{d-1}\times \mathbb{S}^{d-1}}W(x,y)d\sigma(x)d\sigma(y)\leq \lambda^\ast_0 \leq \max_{x\in\mathbb{S}^{d-1}}d(x)
\]
\end{theorem}
\begin{proof}
By Courant-Fisher min-max principle we have \[\lambda^\ast_0=\max_{f\in L^2([-1,1])}{\frac{\langle T_Wf,f\rangle}{\langle f,f\rangle}}\] In particular, if we take the function $\mathbbm{1}(x):=1$ for $x\in [-1,1]$ we have \begin{align*}
    \lambda^\ast_0&\geq \frac{\langle T_W\mathbbm{1},\mathbbm{1} \rangle}{\langle \mathbbm{1},\mathbbm{1} \rangle }\\
    &=\frac{\int_{\mathbb{S}^{d-1}}W(x,y)d\sigma(x)d\sigma(y)}{\int_{\mathbb{S}^{d-1}d\sigma(y)}}\\
    &=d_W
\end{align*}
the last follows form the definition of $d_W$ and the fact that $\sigma$ is a probability measure on the sphere.
On the other hand, if $f_0$ is an eigenfunction associated with $\lambda_0$ we can choose $x^\ast$ such that $f_0(x^\ast)\geq f_0(x)$ for $x\in[-1,1]$. Without loss of generality, assume that $f_0(x^\ast)\neq 0$. So \begin{align*}
  \lambda^\ast_0&=\frac{T_Wf_0(x^\ast)}{f_0(x^\ast)}\\
  &=\int_{\mathbb{S}^{d-1}}W(x^\ast,y)\frac{f_0(y)}{f_0(x^\ast)}d\sigma(y)\\
  &\leq \int_{\mathbb{S}^{d-1}}W(x^\ast,y)d\sigma(y)\\
  &=d_W(x^\ast)
\end{align*}
which finish the proof
\end{proof}
Since $G^\gamma_0(t)=\mathbbm{1}(t)$ we have by Lemma \ref{lem:lambda0largest} and Theorem~\ref{thm:lambda0largest} that the $\mu^\ast_0=\lambda^{\mathrm{sort}}_0$.

\section{Proof of the rate of convergence of the Algorithm}\label{sec:proof}
This section is devoted to the proof of the main theorem, Theorem~2.2. In the sequel, the sentence ``$n$ large enough'' means that $n$ is bigger than some $n_0\geq1$ that may depend on $W$ and $\alpha$. Recall that the result obtained will hold with probability~$1-\alpha$ with $\alpha>0$ arbitrarily small. We used through the paper, the notation $X\leq_{\alpha}C$, where $X$ is random variable and $C$ a constant, to indicate that the inequality holds with probability bigger than $1-\alpha$. 

The aim is to bound $\|\mathcal G^\ast-\hat{\mathcal G}\|_F$ and we will split it into two terms as follows
\[
\|\mathcal G^\ast-\hat{\mathcal G}\|_F\leq\|\mathcal G^\ast-{\mathcal G}\|_F+\|\mathcal G-\hat{\mathcal G}\|_F
\]
where the matrix $\mathcal G$ will be defined later (see Proposition~\ref{prop:good_event}) using a subset of eigenvectors~$V$ of $T_n$. We will treat these terms separately starting with $\|\mathcal G-\hat{\mathcal G}\|_F$ in Section~\ref{proof:step2} and the other term in Section~\ref{proof:step3}.

The first step is to control the probability of the following event $\mathcal E$
\[\mathcal E:=\Big\{\delta_2\Big(\lambda\big(\frac{1}{\rho_n}T_n\big),\lambda(T_W)\Big)
\vee \frac{2^{\frac92}\sqrt d}{\rho_n\Delta^\ast}\|T_n-\hat{T_n}\|_{op}\leq\frac{ \Delta^\ast}8\Big\}\,,\]
where $\Delta^\ast$ is the spectral gap $Gap_1(W)$. We will prove in Section~\ref{proof:step1} that this event holds with probability $1-\alpha/2$ when $n$ is large enough. This event ensures that the ``noise level'' is lower than the spectral gap $Gap_1(W)$ and it guarantees that our algorithm recovers the right subset of eigenvectors as will see in Proposition~\ref{prop:good_event}, Section~\ref{proof:step1}. 

\subsection{Event guaranteeing the algorithm convergence}
\label{proof:step1}
Invoke Theorem~\ref{thm:bandeira_vanhandel} with $Y=\hat{T}_n-T_n$, which by definition have independent centered entries (conditional to latent points $\{X_i\}_{i=1}^n$), to obtain 
\[\mathbb{P}\Big(\|\hat{T}_n-T_n\|_{op}\geq\frac{3\sqrt{2D_0}}{n}+C_0\frac{\sqrt{\log{n}/\alpha}}{n}\Big)\leq \alpha\]
for $\alpha\in(0,1/3)$. Note that for $n$ large enough, one has 
\begin{equation*}
    \|\hat{T}_n-T_n\|_{op}\leq_{\alpha/4}C\max{\big\{\sqrt{\frac{\rho_n}{n}},\frac{\sqrt{\log n}}{n}\big\}}
\end{equation*} 
by Theorem \ref{thm:bandeira_vanhandel}, because $D_0=\max_{0\leq i\leq n}\sum^n_{j=1}\Theta_{ij}(1-\Theta_{ij})$ is $\mathcal{O}(n\rho_n)$. Thus, for $n$ large enough we have
\begin{equation}
\label{eq:BvHeq}
\frac{1}{\rho_n}\|\hat{T}_n-T_n\|_{op}\leq_{\alpha/4}C\max{\big\{\frac{1}{\sqrt{\rho_n n}},\frac{\sqrt{\log n}}{\rho_n n }\big\}}\leq \frac{(\Delta^\ast)^2}{2^{\frac{17}2}\sqrt d},
\end{equation}
provided that $\frac{\sqrt{\log n}}{\rho_n n }=o(1)$, which is the case when $\rho_n=\Omega(\log n/n)$, which we have called the \emph{relatively sparse} case.
Let $V\in\mathbb R^{n\times d}$ and $\hat{V}\in\mathbb R^{n\times d}$ be two matrices with columns corresponding to the eigenvectors associated to eigenvalues $\lambda_{i_1},\lambda_{i_2},\ldots,\lambda_{i_d}$ and $\hat{\lambda}_{i_1},\hat{\lambda}_{i_2},\cdots,\hat{\lambda}_{i_d}$ of $T_n$ and $\hat{T}_n$ respectively, as in Theorem \ref{thm:davis_kahan}. We use Lemma~\ref{lem:matrixineq} with $A=\hat{V}\hat{O}$ and $B=V$, where $\hat{O}$ is an orthogonal matrix, and Theorem \ref{thm:davis_kahan} assuming that the right hand side of~\eqref{eq:thmDK} is smaller than $1$, obtaining 
\begin{align}\|\hat{V}\hat{V}^T-VV^T\|_F&\leq 2\|\hat{V}\hat{O}-V\|_F\nonumber\\ \label{eq:ineq_matproducts} &\leq\frac{2^{\frac{5}{2}}\min{\{\sqrt{d}\|T_n-\hat{T_n}\|_{op},\|T_n-\hat{T_n}\|_{F}\}}}{\Delta} \end{align} where $\Delta:=dist(\{\lambda_{i_1},\cdots,\lambda_{i_d}\},\lambda(T_n)\setminus\{\lambda_{i_1},\cdots,\lambda_{i_d}\})$. 
Then we have \begin{align}
\|\hat{V}\hat{V}^T-VV^T\|_F& \leq \frac{2^{\frac{5}{2}}\frac{\sqrt{d}}{\rho_n}\|T_n-\hat{T}_n\|_{op}}{\frac{1}{\rho_n}\Delta}\nonumber\\\label{eq:aproxVVT}
&\leq_{\alpha} \frac{\rho_n(\Delta^\ast)^2}{2^6\Delta}
\end{align}
Now, we use the $\delta_2$ metric to quantify the convergence of the eigenvalues of the normalized probability matrix $\frac{1}{\rho_n}T_n$ to the eigenvalues of the integral operator $T_W$. From Theorem \ref{thm:prop4} we have that, when $n$ is large enough
\begin{equation}
\label{eq:delta2_event}
    \delta_2\Big(\lambda(\frac{1}{\rho_n}T_n),\lambda(T_W)\Big)\leq_{\alpha/4} C\Big(\frac{\log{n}}{n}\Big)^{\frac{s}{2s+d-1}}\leq\frac{ \Delta^\ast}8\,,
\end{equation}
where $\Delta^\ast$ is the spectral gap $Gap_1(W)$. This and \eqref{eq:BvHeq} ensure that $\mathcal E$ has probability $1-\alpha/2$. In particular, it gives the following result proving that our algorithm find the right eigenvectors.

\begin{proposition}
\label{prop:good_event}
On the event $\mathcal E$, there exists one and only one set $\Lambda_1$ of $d$ eigenvalues of $\frac{1}{\rho_n}\hat{T_n}$ separated by at least $\Delta^\ast/2$ from the other eigenvalues of $\hat{T}_n$. These eigenvalues are at a distance at most $\Delta^\ast/8$ of $\frac{1}{\rho_n}\lambda_1,\ldots,\frac{1}{\rho_n}\lambda_d$, the eigenvalues of $T_n$ whose eigenvectors define the matrix $\mathcal G:=(1/c_1)VV^T$. Furthermore, on the event $\mathcal E$, our algorithm returns the  matrix $\hat{\mathcal G}=(1/c_1)\hat V\hat V^T$ composed by the eigenvectors corresponding to the eigenvalues of $\Lambda_1$.
\end{proposition}
\begin{proof}
When $\Delta^\ast>0$, we remark that $\lambda_1^*=\lambda_2^*=\ldots=\lambda_d^*$ is the only eigenvalue of $T_W$ with multiplicity $d_1=d$, the others eigenvalues (except for $\lambda_0^\ast$) having multiplicity strictly greater than~$d$. Now, using \eqref{eq:delta2_event} we deduce that there exists a unique set $\frac{1}{\rho_n}\lambda_{i_1},\frac{1}{\rho_n}\lambda_{i_2},\ldots,\frac{1}{\rho_n}\lambda_{i_d}$ of $d$ eigenvalues of $T_n$ that can be separated from the other eigenvalues by a distance  at least $3\Delta^\ast/4$, namely the triangular inequality gives
\begin{equation}
\label{eq:empirical_gap}
    \frac{\Delta}{\rho_n}\geq\frac{3\Delta^\ast}4\,.
\end{equation}
To these eigenvalues correspond the eigenvectors $V\in\mathbb R^{n\times d}$ defining $\mathcal G:=(1/c_1)VV^T$. 

Furthermore, using \eqref{eq:aproxVVT} we get that there exists eigenvalues $\hat\lambda_{i_1},\hat\lambda_{i_2},\ldots,\hat\lambda_{i_d}$ and eigenvectors $\hat V\in\mathbb R^{n\times d}$ of $\hat{T_n}$ such that $\|\hat{V}\hat{V}^T-VV^T\|_F\leq \Delta^\ast /48$. We define $\Lambda_1:=\{\hat{\lambda}_{i_1},\cdots,\hat{\lambda}_{i_d}\}$. By Hoffman-Wielandt inequality \cite[Thm.VI.4.1]{Bathia}, it holds
\[
\Big(\sum_{k=1}^d(\hat\lambda_k^{\mathrm{sort}}-\lambda_k^{\mathrm{sort}})^2\Big)^{1/2}\leq \|\hat{V}\hat{V}^T-VV^T\|_F\leq \Delta^\ast /8\,,
\]
where $\hat\lambda^{\mathrm{sort}}_1\geq\hat\lambda^{\mathrm{sort}}_{2}\geq\cdots\geq\hat\lambda_{d}^{\mathrm{sort}}$ (resp. $\lambda^{\mathrm{sort}}_1\geq\lambda^{\mathrm{sort}}_{2}\geq\cdots\geq\lambda_{d}^{\mathrm{sort}}$) is the sorted version of the eigenvalues $\hat\lambda_{i_1},\cdots,\hat\lambda_{i_d}$ (resp. $\lambda_{i_1},\cdots,\lambda_{i_d}$). By triangular inequality, we deduce that 
\[
\hat\Delta:=dist(\Lambda_1,\lambda(\hat T_n)\setminus\Lambda_1)
\geq \frac{\Delta^\ast}{2}\,,
\]
namely $\hat\lambda_{i_1},\hat\lambda_{i_2},\ldots,\hat\lambda_{i_d}$ is a set of $d$ eigenvalues at distance at least ${\Delta^\ast}/{2}$ from the other eigenvalues of $\hat{T}_n$.


This analysis can be also done for the other eigenvalues as follows. Eq. \eqref{eq:delta2_event} shows that there exists a set of $d_k$ eigenvalues of $T_n$ which concentrate around $\mu^\ast_k$, and such that it has diameter smaller than ${\Delta^\ast}/{4}$. Recall that $d_k$ is the size of the Spherical Harmonics space $k$ and $d_k>d_1=d$. Weyl's inequality \cite[P.63]{Bathia} shows that there exists a set $\Lambda_k$ of $d_k$ eigenvalues of $\hat T_n$ around $\mu^\ast_k$ of size ${\Delta^\ast}/{4}$. Now, consider a subset $L$ of $d$ eigenvalues which is different from $\Lambda_1$ then the previous discussion shows that there exists an eigenvalue $\hat\lambda$ which is not in $L$ and that belongs to same cluster to one of the eigenvalues in $L$. In particular $\hat\lambda$ is at a distance less than ${\Delta^\ast}/{4}$ of $L$. By \eqref{eq:empirical_gap} we deduce that, on the event $\mathcal E$, Algorithm 1 returns $\hat{\mathcal G}=(1/c_1)\hat V\hat V^T$ composed by the eigenvectors corresponding to the eigenvalues of the aforementioned cluster of $d$ eigenvalues.
\end{proof}
We now prove the following lemma, which is stated in the article
\begin{lemma}\label{lem:eq.defsgap}
On the event $\mathcal{E}$, the following equality holds
\[\operatorname{Gap}_1(\hat{T}_n)=\max{\Big\{\max_{1\leq i\leq n-d}{\min{\{\mathrm{left}(i),\mathrm{right}(i+d)\}}},\mathrm{left}(n-d+1)\Big\}}\]
\end{lemma}
\begin{proof}
The lemma follows from Proposition \ref{prop:good_event}. Indeed, on the event $\mathcal{E}$ there exist only one set $\Lambda_1$ of eigenvalues of $\hat{T}_n$ with cardinality $d$ , whose distance to the rest of the spectrum is larger that $\Delta^\ast$ and its diameter is smaller that $\Delta^\ast$. When sorting the eigenvalues of $\hat{T}_n$ in decreasing order, those belonging to $\Lambda_1$ will appear in consecutive order. The lemma follows from this observation and from the fact $\operatorname{Gap}_1(\hat{T}_n ;i_{n-d-1},\cdots,i_{n-1})=\mathrm{left}(n-d-1)$.
\end{proof}

\subsection{Sampling error control}
\label{proof:step2}
We have by \eqref{eq:ineq_matproducts} that
\begin{equation}
\label{eq:first_term_bound}
    \|\hat{\mathcal G}-\mathcal G\|_F=\frac{1}{c_1}\|\hat{V}\hat{V}^T-VV^T\|_F\leq_\alpha C\frac{(d-2)}{\sqrt{d{n}}}\,,
\end{equation}
whenever $n$ is large enough and $\Delta^\ast>0$, where $C$ may depend on $W$. In the last inequality we used that $c_1=d/(d-2)$.
\subsection{Sampled eigenvectors convergence}
\label{proof:step3}
We are left to control $\|\mathcal G^\ast-{\mathcal G}\|_F$. We begin by recalling some basic definitions we have made through the paper and introducing some notation. Set $R=\mathcal O((n/\log n)^{\frac1{2s+d-1}})$ and $\tilde R:=d_0+d_1+\ldots+d_R$ the total size of the $R+1$ first Harmonic spaces. It is well known that $\tilde R=\mathcal O(R^{d-1})=o(n)$ for $s>0$. If $W_R$ is the rank $R^\prime$ approximation of $W$, we have \[T_R=\big(\frac1n W_R(X_i,X_j)\big)_{i,j}=\Phi_{0,R}\Lambda^\ast_{0,R}\Phi_{0,R}\] where 
$\Phi_{0,R}$ is the matrix with columns $\Phi_k\in\mathbb{R}^n$, for $0\leq k\leq R^\prime$, such that $(\Phi_k)_i=\phi_k(X_i)$ and $\Lambda^\ast_{0,R}=diag(\lambda^\ast_{0},\lambda^\ast_{2},\cdots,\lambda^\ast_{\tilde R})$. Similarly $\Lambda_{0,R}=diag(\lambda_{0},\lambda_{2},\cdots,\lambda_{\tilde R})$. Let $\tilde{V}$ be the matrix that contains as columns the eigenvectors of the matrix $T_n$ and $\tilde{V}_R$ contains as columns the eigenvectors $T_R$ so we have the eigenvalue decomposition 
\[T_n=\tilde{V}\Lambda\tilde{V}^T\]  \[T_R=\tilde{V}_R\Lambda_R{\tilde{V}_R}^T\] 
Let $V$ be the matrix that contains the columns $1,\cdots,d$ of $\tilde{V}$, $V_R$ contains the columns $1,\cdots,d$ of $\tilde{V}_R$ and $V^\ast$ contains $\phi_k$ for $1\leq k\leq d$ as columns. Then $\mathcal G^\ast,\mathcal G,\mathcal G_R,\mathcal G^\ast_{proj}$ are defined by
\begin{align*}
\mathcal G^\ast:&=\frac{1}{c_1}V^\ast(V^\ast)^T\\
 \mathcal G:&=\frac{1}{c_1}VV^T\\
\mathcal G_R:&=\frac{1}{c_1}V_R{V_R}^T\\
\mathcal G^\ast_{proj}:&=V^\ast({V^\ast}^TV^\ast)^{-1}{V^\ast}^T
\end{align*}

Note that $\mathcal{G}^\ast_{proj}$ is the projection matrix for the column span of the matrix $V^\ast$, that is, it is the projection matrix onto the space $\operatorname{span}\{\Phi_1,\cdots,\Phi_d\}$.

We have by triangle inequality
\[
    \|\mathcal{G}^\ast-\mathcal{G}\|_F\leq \|\mathcal{G}^\ast-\mathcal{G}_{proj}^\ast\|_F+\|\mathcal{G}_{proj}^\ast-\mathcal{G}_R\|_F+\|\mathcal{G}_R-\mathcal{G}\|_F
\]

We call truncation error to the last term in the right hand side, because it is related to the fact that $W_R$ is a rank $R^\prime$ approximation of $W$. 

To bound $\|\mathcal{G}-\mathcal{G}_R\|_F$ we will use Theorem \ref{thm:davis_kahan} noting that $\mathcal{G}$ and $\mathcal{G}_R$ have as columns the eigenvectors of matrices $T_n$ and $T_R$. So \[\|\mathcal{G}-\mathcal{G}_R\|_F\leq \frac{2^{\frac{3}{2}}\|T_n-T_R\|_F}{\Delta}\leq C\frac{ (n/\log n)^{-s/(2s+d-1)}}{\Delta}\]

where we recall that $R=\mathcal O((n/\log n)^{\frac1{2s+d-1}})$, which gives the optimal rate for this error term \cite{Yohann}.  In order to bound $\|\mathcal{G}^\ast-\mathcal{G}^\ast_{proj}\|_F$ we use Lemma \ref{lem:projapprox} with $B=V^\ast$ obtaining \[\|\mathcal{G}^\ast-\mathcal{G}^\ast_{proj}\|_F\leq \|\mathrm{Id}_d-{V^\ast}^TV^\ast\|_F\]
On the other hand, we have \begin{align*}
\|\mathrm{Id}_d-{V^\ast}^TV^\ast\|_F&\leq\sqrt{d}\|\mathrm{Id}_d-{V^\ast}^TV^\ast\|_{op} \\
&\leq_\alpha \frac{d}{\sqrt{n}}
\end{align*}
where we used Theorem \ref{thm:gramconcen} to obtain the last inequality.

It only remains to bound the term $\|\mathcal{G}^\ast_{proj}-\mathcal{G}_R\|_F$. We concentrate first in bounding the term $\mathcal{G}^\ast_{proj}\mathcal{G}^\perp_R$. We use Theorem \ref{thm:perturbeigensp}, with $E=\mathcal{G}^\ast_{proj}$, $F=\mathcal{G}^\perp_R$, $B=T_R$ and $A=T_R+H$, where
\[
H:=\tilde{\Phi}_{0,R}\Lambda^\ast_{0,R} \tilde{\Phi}_{0,R}^T-\Phi_{0,R}\Lambda^\ast_{0,R}\Phi_{0,R}
\]
 the matrix $\tilde{\Phi}_{0,R}$ has column $\tilde{\Phi}_k$ for $k\in \{1,\cdots,R^\prime\}$ where the $\tilde{\Phi}_k$ are obtained from $\Phi_k$ by a Gram-Schmidt orthonormalization process. In other words, there exists a matrix $L$ such that $\tilde{\Phi}_{0,R}=\Phi_{0,R}(L^{-1})^T$. The matrix $L$ comes from the Cholesky decomposition of $\Phi_{0,R}^T\Phi_{0,R}$, that is, $L$ satisfy $\Phi_{0,R}^T\Phi_{0,R}=LL^T$. 
 
Note that $A$ and $B$ are symmetric, hence normal matrices, so Theorem \ref{thm:perturbeigensp} applies. Also, in the event $\mathcal{E}$, we can take $S_1=(\lambda_1-\frac{\Delta^\ast}{8},\lambda_1+\frac{\Delta^\ast}{8})$ and 
$S_2~=~\mathbb{R}~\setminus~(\lambda_1-~\frac{7\Delta^\ast}{8},\lambda_1+\frac{7\Delta^\ast}{8}))$. By Theorem \ref{thm:perturbeigensp} we have 
\[\|\mathcal{G}^\ast_{proj}\mathcal{G}_R^\perp\|_F\leq \frac{\|A-B\|_F}{\Delta^\ast}=\frac{\|H\|_F}{\Delta^\ast}\]
where $\Delta:=\min_{k,\ell\neq 1,\ldots, d}{\{|\lambda^\ast_{k}-\lambda^\ast_{1}|,|\lambda^\ast_{d}-\lambda^\ast_{\ell}|\}}$.
It remains to bound $H$.

We have that 
\begin{align*}
\|H\|_F&\leq\|L^{-T}\Lambda^\ast_{0,R}L^{-1}-\Lambda^\ast_{0,R}\|_F
\|\Phi_{0,R}^T\Phi_{0,R}\|_{op}\\
&\leq \|\Lambda^\ast_{0,R}\|_F \|L^{-1}L^{-T}-\mathrm{Id_{R^\prime}} \|_{op}\|\Phi_{0,R}^T\Phi_{0,R}\|_{op}
\end{align*}
where in the last line we used Corollary \ref{cor:ostrowski}. It is easy to see that \[\|L^{-1}L^{-T}-\mathrm{Id_{R^\prime}}\|_{op}=\|(\Phi_{0,R}^T\Phi_{0,R})^{-1}-\mathrm{Id_{R^\prime}}\|_{op}\]
which, using \cite[Lem.12]{Yohann}, implies that\[\|Z\|_F\leq_{\alpha/4}2C_1\frac{R^{d-1}}{\sqrt{n}}\]
which, since $R=\mathcal O((n/\log n)^{\frac1{2s+d-1}})$, becomes 
\[\|Z\|_F\leq_{\alpha/4} C^\prime\Big(\frac{\log{n}}{n}\Big)^{\frac{s}{2s+d-1}} \]
for a constant $C^\prime>0$. Collecting terms we obtain \[\|\mathcal{G}^\ast_{proj}\mathcal{G}_R^\perp\|_F\leq_{\alpha/4} \frac{C^{\prime\prime}}{\Delta^\ast}\Big(\frac{\log{n}}{n}\Big)^{\frac{s}{2s+d-1}} \]

Since $\mathcal{G}^\ast_{proj}$ and $\mathcal{G}_R$ are projectors we have, see \cite[p.202]{Bathia}\[\|\mathcal{G}^\ast_{proj}-\mathcal{G}_R\|_F=2\|\mathcal{G}^\ast_{proj} {\mathcal{G}_R}^\perp \|_F\]
which implies that \[\|\mathcal{G}^\ast_{proj}-\mathcal{G}_R\|_F\leq_{\alpha/4} \frac{2C^{\prime\prime}}{\Delta^\ast}(\frac{n}{\log{n}})^{\frac{-s}{2s+d-1}}\]

To conclude, we have that \begin{align*}
\|\mathrm{Id}_d-{V^\ast}^TV^\ast\|_F&\leq\sqrt{d}\ \|\mathrm{Id}_d-{V^\ast}^TV^\ast\|_{op} \\
&\leq_{\alpha/4} \frac{d}{\sqrt{n}}
\end{align*}
where we use Theorem \ref{thm:gramconcen} in the second inequality. Collecting terms we conclude that\[\|\mathcal{G}^\ast-\mathcal{G}\|_F\leq_{\alpha/4}\frac{C_d}{\Delta^\ast}\Big(\frac{\log{n}}{n}\Big)^{\frac{s}{2s+d-1}}\] 
where $C_d$ is a constant that depends on $d$ and $\alpha$.
\section{Useful results}
\begin{lemma}\label{lem:matrixineq}
Let $A$, $B$ be two matrices in $\mathbbm{R}^{n\times d}$ then 
\begin{align*}
    \|AA^T-BB^T\|_F&\leq (\|A\|_{op}+\|B\|_{op})\|A-B\|_F\\
    \|AA^T-BB^T\|_{op}&\leq (\|A\|_{op}+\|B\|_{op})\|A-B\|_{op}\,.
\end{align*}
If it holds that $A^TA=B^TB=I_d$ then 
\begin{align*}
    \|AA^T-BB^T\|_F&\leq 2\|A-B\|_F
\end{align*}
\end{lemma}
\begin{proof}
We begin with the first inequality
\begin{align*}
    \|AA^T-BB^T\|_F & = \|(A-B)A^T+B(A^T-B^T)\|_F\\
    & \leq \|(A\otimes I_n)\mathrm{vec}(A-B)\|_2+ \|(I_d\otimes B)\mathrm{vec}(A-B)^T\|_2 \\
    & \leq (\|A\otimes I_n\|_{op}+\|I_d\otimes B\|_{op})  \|A-B\|_F\\
    & = (\|A\|_{op}+\|B\|_{op})\|A-B\|_F\,.
\end{align*}
Here $\mathrm{vec}(\cdot)$ represent the vectorization of a matrix, that its transformation into a column vector. 
The second inequality is given by 
\begin{align*}
    \|AA^T-BB^T\|_{op} & = \|(A-B)A^T+B(A^T-B^T)\|_{op}\\
    & \leq  (\|A\|_{op}+\|B\|_{op})\|A-B\|_{op}\,.
\end{align*}
The third statement is an elementary consequence of the above inequalities.
\end{proof}

\begin{lemma}\label{lem:projapprox}
Let $B$ a $n\times d$ matrix with full column rank. Then we have \[\| BB^T-B(B^TB)^{-1}B^T\|_F= \|\mathrm{Id}_d-B^TB\|_F\]
\end{lemma}
\begin{proof}
We have \begin{align*}
    \| BB^T-B(B^TB)^{-1}B^T\|_F&=\|B\big((B^TB)^{-1}-\mathrm{Id}_d\big)B^T\|_F\\
\end{align*}
and by definition of the Frobenious norm and cyclic property of the trace
\begin{align*}
    \|B\big((B^TB)^{-1}-\mathrm{Id}_d\big)B^T\|^2_F&=tr\big(B((B^TB)^{-1}-\mathrm{Id}_d)B^T B((B^TB)^{-1}-\mathrm{Id}_d)B^T\big)\\
    &=tr\big((\mathrm{Id}_d-B^TB)^2\big)\\
    &=\|\mathrm{Id}_d-B^TB\|^2_F
\end{align*}
\end{proof}

\subsection{Bandeira-Van Handel theorem}\label{sec:BanVan}
The following theorem is a slight reformulation of the \cite[Cor.3.12]{BanVan} 
\begin{theorem}[Bandeira-Van Handel]\label{thm:bandeira_vanhandel}
Let $Y$ be a $n\times n$ symmetric random matrix whose entries $Y_{ij}$ are independent centered random variables. There exists a universal constant $C_0$ such that for $\alpha\in (0,1)$ \[\mathbb{P}\Big(\|Y\|_{op}\geq 3\sqrt{2D_0}+C_0\sqrt{\log{n}/\alpha}\Big)\leq \alpha\]
where $D_0=\max_{0\leq i\leq n}\sum^n_{j=1}Y_{ij}(1-Y_{ij})$.
\end{theorem}
\begin{proof}
By \cite[Rmk.3.13]{BanVan} we have the tail concentration bound (taking their $\epsilon$ equal to $1/2$) \[\mathbb{P}\Big(\|Y\|_{op}\Big)\geq 3\sqrt{2D_0}+\max_{ij}{|Y_{ij}|}C_0\sqrt{\log{n/\alpha}}\] 
the result follows, because $\max_{ij}{|Y_{ij}|}\leq 1$.
\end{proof}
Using the previous theorem with $Y=\hat{T}_n-T_n$, which is centered and symmetric, we obtain the tail bound \[\mathbb{P}\Big(\|\hat{T}_n-T_n\|_{op}\geq  \frac{3\sqrt{2D_0}}{n}+C_0\frac{\sqrt{\log{n}/\alpha}}{n}\Big)\leq \alpha\]
\subsection{Davis-Kahan $\sin {\theta} $ theorem}\label{sec:sintheta}
For $n$ large enough, the eigenspace associated to the eigenvalue $\hat{\lambda}_1$ is close to the eigenspace associated to the eigenvalue $ \lambda_1$. This is precised by the Davis-Kahan $sin$ $\theta$ theorem. We use the following version which is proved in \cite{YuWanSam}
\begin{theorem}\label{thm:davis_kahan}
Let $\Sigma$ and $\hat{\Sigma}$ be two symmetric $\mathbbm{R}^{n\times n}$ matrices with eigenvalues $\lambda_1\geq \lambda_2\geq \cdots \geq \lambda_n$ and  $\hat{\lambda}_1\geq \hat{\lambda}_2 \geq \cdots \hat{\lambda}_n $ respectively. For $1\leq r\leq s\leq n$ fixed, we assume that $\min{\{\lambda_{r-1}-\lambda_r, \lambda_s-\lambda_{s-1}\}}>0$ where $\lambda_0:=\infty$ and $\lambda_{n+1}=-\infty$. Let $d=s-r+1$ and $V$ and $\hat{V}$ two matrices in $\mathbbm{R}^{n\times d}$ with columns $(v_r,v_{r+1},\cdots,v_s)$ and $(\hat{v}_r,\hat{v}_{r+1},\cdots,\hat{v}_s)$ respectively, such that $\Sigma v_j=\lambda_j v_j$ and $\hat{\Sigma}\hat{v}_j=\hat{\lambda}_j \hat{v}_j$. Then there exists an orthogonal matrix $\hat{O}$ in $\mathbbm{R}^{d\times d}$ such that \begin{equation}\label{eq:thmDK}\|\hat{V}\hat{O}-V\|_{F}\leq \frac{2^{3/2}\min{\{\sqrt{d}\|\Sigma-\hat{\Sigma}\|_{op},\|\Sigma-\hat{\Sigma}\|_{F}\}}}{\min{\{\lambda_{r-1}-\lambda_{r},\lambda_s-\lambda_{s+1}\}}}\end{equation}
\end{theorem}

Also, we need the following perturbation result \cite[Thm.VII.2.8]{Bathia}\label{thm:pertubeig}
\begin{theorem}
\label{thm:sylvester}
Let $A$ and $B$ two the normal matrices and define $\delta=dist(\lambda(A),\lambda(B))$. If $X$ satisfies the Sylvester equation $AX-XB=Y$, then \[\|X\|_F\leq \frac{1}{\delta}\|Y\|_F\]
\end{theorem}
Another useful perturbation theorem \cite[Thm.VII.3.1]{Bathia}
\begin{theorem}\label{thm:perturbeigensp}
Let $A$ and $B$ be two normal operators and $S_1$ and $S_2$ two sets separated by a strip of size $\delta$. Let $E$ be the orthogonal projection matrix of the eigenspaces of $A$ with eigenvalues inside $S_1$ and $F$ be the orthogonal projection matrix of the eigenspaces of $B$ with eigenvalues inside $S_2$. Then 
\[\|EF\|_F\leq \frac{1}{\delta}\|E(A-B)F\|_F\leq \frac{1}{\delta}\|A-B\|_F\]
\end{theorem}
\subsection{Ostrowski theorem}
The following eigenvalue perturbation theorem is due to Ostrowski \cite[Thm.4.5.9]{Horn} and \cite[Cor.3.54]{Braunphd}
\begin{theorem}
Let $A\in\mathbb{R}^{n\times n}$ be a Hermitian matrix and $S\in\mathbb{R}^{n\times n}$ be a nonsingular matrix. Then for each $1\leq i\leq n$ there exists $\theta_i>0$ such that \[\lambda_i(S A S^\ast)=\theta_i\lambda_i(A)\]
In addition, it holds \[|\lambda_i(SAS^\ast)-\lambda_i(A)|\leq |\lambda_i(A)|\|S^\ast S-\mathrm{Id_n}\|_{op}\]
\end{theorem}
\begin{remark}
The previous theorem is also valid for $S$ singular \cite[Cor.4.5.11]{Horn}.
\end{remark}
The previous theorem can be extended to the case where $S$ is not necessarily a square matrix \cite[Cor.3.59]{Braunphd}
\begin{cor}\label{cor:ostrowski}
Let $A\in \mathbb{R}^{n\times n}$ be a Hermitian matrix and $S\in \mathbb{R}^{d\times n}$ matrix then \[|\lambda_i(SAS^\ast)-\lambda_i(A)|\leq |\lambda_i(A)|\|S^\ast S-\mathrm{Id_n}\|_{op}\]  
\end{cor}
From the previous result we deduce the following corollary
\begin{cor}\label{cor:ostfrob}
Under the same conditions of Corollary \ref{cor:ostrowski} we have \[\|SAS^\ast-A\|_F\leq \|A\|_F\|S^\ast S-\mathrm{Id_n}\|_{op}\]
\end{cor}
\subsection{Convergence rate of regular graphon estimation}
We use the following result, which can be found in \cite{Yohann}
\begin{theorem}\label{thm:prop4}
Let $W$ be a graphon on the sphere of the form $W(x,y)=f(\langle x,y \rangle)$. If $f$ belongs to the weighted Sobolev space $Z^s_{w_\gamma}\big((-1,1)\big)$ then we have \[\delta_2(\lambda\big(\frac{1}{\rho_n}T_n\big),\lambda(T_W))\leq_{\alpha} C\Big(\frac{\log{n}}{n}\Big)^{\frac{s}{2s+d-1}}\]
where $\leq_\alpha$ means that the inequality holds with probability greater than $1-\alpha$ for $\alpha\in (0,1/3)$ and $n$ large enough.
\end{theorem}
While Theorem \ref{thm:bandeira_vanhandel} gives a bound for the difference of the eigenvalues of the observed matrix with respect to the eigenvalues of the probability matrix, Proposition \ref{thm:prop4} ensures that the eigenvalues of the empirical matrix are close to these of the integral operator. 
\subsection{Covariance matrix approximation}
Given a set of independent random vectors $X_1,\cdots,X_n$ uniformly distributed on the sphere $\mathbb{S}^{d-1}$ we are interested in the concentration properties of the quantity $\frac{1}{n}\sum^n_{k=1}X_iX_i^T$ around its mean, which is $\mathbb{E}(X_iX_i^T)=\mathrm{Id}_d$ for $1\leq i\leq n$ (in other words, the vectors $X_i$ are isotropic). Since the uniform distribution on the sphere is sub-gaussian \cite[Thm.3.4.6]{Vers}, we can use the following theorem \cite[Prop.2.1]{Vershy}.
\begin{theorem}\label{thm:gramconcen}
If $X_1,\cdots,X_n$ are independent random vectors in $\mathbb{R}^{d}$ with $d\leq n$ which have sub-gaussian distribution. Then for any $\alpha\in (0,1)$ it holds 
\[\big\|\frac{1}{n}\sum^n_{k=1}X_kX_k^T-\mathrm{Id}_d\big\|_{op}\leq_\alpha \sqrt{\frac{d}{n}}\]

\end{theorem}

\end{document}